\documentclass[11pt]{article}
\usepackage[margin=1in]{geometry}
\usepackage[T1]{fontenc}
\usepackage{amsmath,amsthm}
\usepackage{newtxtext,newtxmath,microtype}
\usepackage{booktabs,graphicx,array,enumitem}
\usepackage[numbers,sort&compress]{natbib}
\usepackage{xcolor}
\usepackage[colorlinks=true,linkcolor=blue!45!black,citecolor=blue!45!black,urlcolor=blue!45!black]{hyperref}
\usepackage{fancyhdr}
\usepackage[font=small,labelfont=bf]{caption}
\usepackage{placeins,listings}
\graphicspath{{figures/}}
\newtheorem{theorem}{Theorem}
\newtheorem{proposition}[theorem]{Proposition}
\newtheorem{corollary}[theorem]{Corollary}
\newcommand{\R}{\mathbb{R}}
\newcommand{\E}{\mathbb{E}}
\newcommand{\one}{\mathbf{1}}
\newcommand{\tr}{\operatorname{tr}}
\newcommand{\diag}{\operatorname{diag}}
\newcommand{\norm}[1]{\left\lVert#1\right\rVert}

\setlist{nosep}
\title{\vspace{-1em}\textbf{Contingent Exposure Routing for Financial AI:}\\
Outage Risk and the Cost of Indivisible Decisions}
\author{Shivam Gupta\\
\small Independent Researcher\\
\small \href{mailto:shivam1720406@gmail.com}{shivam1720406@gmail.com}}
\date{23 September 2026}
\begin{document}
\maketitle
\begin{abstract}
Model failover restores availability, but changes which financial institutions share decision errors. We formulate outage-contingent routing through a local market-impact response matrix and study expected squared price displacement. A symmetric construction shows that a shared backup can leave an order-one concentration floor as the number of primary endpoints grows, while balanced fallback risk decreases inversely with the surviving endpoint count. For indivisible decisions, we derive the exact second moment of independent randomized routing and an effective-exposure granularity that determines its gap from fractional allocation. Conditional-expectation rounding gives a finite-agent bound without coupled quotas; a separate swap procedure preserves endpoint counts and is assessed against dual lower bounds. Across 60 synthetic portfolio networks and 11,340 scenario evaluations, the latter reduces risk by 6.57\% and 10.53\% for single and double endpoint removals at the central feedback setting with independent errors. A replay of 1,024 recorded API responses on constructed rebalancing tasks gives a smaller held-out reduction of 3.30\% (paired bootstrap interval 2.07--4.57\%). Strongly aligned errors, inferior endpoints, and indivisibility limit diversification. The contribution is an auditable routing stress test and implementation analysis, not an estimate of real-market crash probabilities.
\end{abstract}

\section{Introduction}
A financial application can remain available while its economic dependencies change sharply. Suppose institutions use different model endpoints to interpret a common input. When some endpoints become unavailable, their requests move to backups. If those policies converge on the same endpoint, a previously dispersed population may act on a common error precisely during a disruption. An availability dashboard can report successful requests throughout this transition.

The general concern is established. The Financial Stability Board identifies third-party dependencies, market correlations, and model risk as potential channels through which AI can affect financial stability \citep{fsb2024}. Algorithmic monoculture can have collective consequences not predicted by individual accuracy \citep{kleinberg2021}. The question here is more operational: \emph{which surviving endpoint should receive each displaced financial decision, and what risk does the resulting executable assignment create?}

Two distinctions matter. First, the relevant load is financial exposure propagated through market impact, rather than request count. Two requests can consume similar inference resources but induce different price responses. Second, allocating half a decision to each of two models is not equivalent to selecting one model with probability one half. The latter introduces assignment variance. A portfolio-style relaxation can therefore understate the risk of the executed policy.

We call the framework \emph{Contingent Exposure Routing} (CER). It combines scenario-specific eligibility, a financial response matrix, and a second moment of endpoint errors. Our contributions are:
\begin{enumerate}
\item A sharp outage construction separating normal from contingent concentration, including the limiting effect of a common error component.
\item An exact finite-population correction for randomized routing, an effective-exposure granularity identity, and examples where the fractional gap persists as the population grows.
\item An executable procedure with explicit constraints, analytical and numerical checks, paired network experiments, and a separately delimited API replay.
\end{enumerate}
These concern the formulation and implementation of a financial routing problem. We do not claim new portfolio variance algebra, a new general rounding principle, or discovery of AI-related systemic risk.

\begin{figure}[t]
\centering\includegraphics[width=\linewidth]{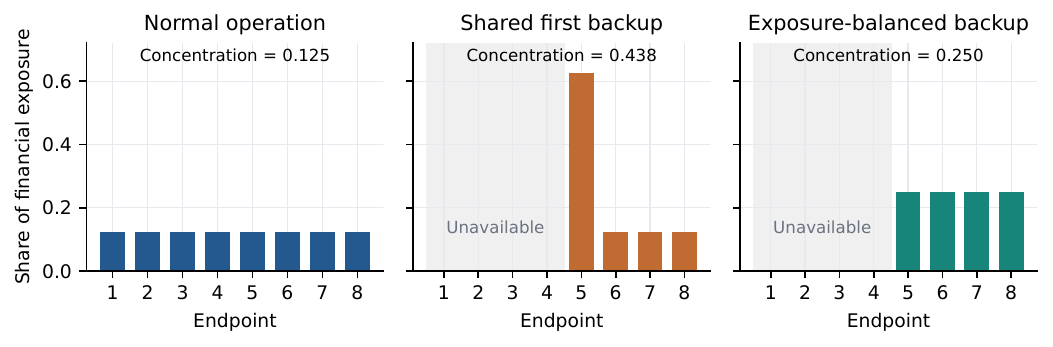}
\caption{Eight equally weighted primary endpoints, followed by removal of four. A common surviving backup concentrates exposure at endpoint 5; balanced redistribution attains the continuous minimum. Labels carry no quality ranking; errors are centered and independent with unit variance. The illustration concerns financial exposure, not inference throughput.}
\label{fig:mechanism}
\end{figure}

\section{Related work}
\paragraph{Financial propagation.}
Overlapping portfolios and deleveraging link institutions. \citet{greenwood2015} quantify vulnerability through balance sheets and common exposures. \citet{cont2017} develop systemic stress testing using liquidity-weighted overlaps, nonlinear deleveraging, and indirect exposures. Our response matrix is a local linear construction motivated by this literature; it does not reproduce their complete stress test or inherit their empirical calibration.

\paragraph{AI concentration and disruption.}
\citet{kleinberg2021} analyze collective consequences of a common algorithm in matching markets. The FSB discusses operational concentration and correlated financial behavior \citep{fsb2024}. Recent preprints examine AI-vendor compromise and financial contagion \citep{leytes2026}, and interactions between AI adoption, herding, and performative feedback \citep{meng2026}. We study the assignment intervention following endpoint removal. Cyber propagation, endogenous adoption, strategic equilibrium, and incident frequencies are outside our model.

\paragraph{Model routing.}
RouteLLM learns cost-quality routing from preference data \citep{ong2024}. RACER calibrates model sets to control a misrouting event \citep{hao2026}. Robust batch-level routing jointly manages assignment, uncertain performance, and resource constraints \citep{markovic2026}. These establish that routing can be learned, risk-aware, and jointly constrained. Our objective differs: errors assigned to different institutions interact through their financial exposures. Individual response quality and model-call counts do not determine it.

\paragraph{Discrete implementation.}
Randomized rounding and derandomization are classical \citep{raghavan1987}; dependent rounding can preserve cardinality constraints \citep{gandhi2006}. Convex relaxations and supporting-hyperplane bounds are standard \citep{boyd2004}. We specialize these ideas to the financial second moment, derive the exact correction for independent endpoint selection, and distinguish its guarantee from the separate quota heuristic. A scoped search through 23 September 2026 did not locate this complete formulation and evaluation. It does not establish that every component or equivalent formulation is absent from prior work.

\section{Financial exposure under contingent routing}
\subsection{Local price response}
There are \(n\) institutions, \(d\) assets, and \(K\) model endpoints. An endpoint can be a model version or an operationally distinct service; endpoint diversity does not imply independent vendors. Let \(E\in\R^{d\times n}\) map relative decision errors to first-round price displacement, and \(B\in\R^{d\times d}\) describe a local feedback response. For institution-level error \(u\),
\begin{equation}
 x=Eu+Bx,\qquad x=(I-B)^{-1}Eu,\qquad \rho(B)<1.
 \label{eq:response}
\end{equation}
For an asset-weighting matrix \(W\succeq0\), define
\begin{equation}
 A=W^{1/2}(I-B)^{-1}E=[a_1,\ldots,a_n].
 \label{eq:A}
\end{equation}
The outcome is \(\norm{W^{1/2}x}_2^2\), a quadratic price-displacement measure. Choosing \(W=ww^\top\) gives the squared displacement of a specified linear portfolio. Neither choice is automatically a default probability, realized monetary loss, or expected shortfall.

Equation~\eqref{eq:response} is a local response around a fixed operating state. Signed deviations may represent excess sales or purchases. Thresholds, insolvency, endogenous liquidity withdrawal, and large-shock saturation need a different propagation model. A participation factor \(\eta\) multiplying \(E\) multiplies all risks by \(\eta^2\); the experiments report normalized units with this factor divided out.

\subsection{Errors, outages, and executable assignments}
In outage state \(s\), the endpoint error vector \(\epsilon_s\in\R^K\) has finite \emph{uncentered} second moment
\begin{equation}
 S_s=\E[\epsilon_s\epsilon_s^\top\mid s]\succeq0.
 \label{eq:S}
\end{equation}
Systematic bias is retained; cross-endpoint dependence is permitted. Each endpoint error is shared by institutions assigned to that endpoint during the decision window. This is a common-input or broadcast-advice model, not a general law for separately prompted agents.

A deterministic matrix \(Z_s\in\{0,1\}^{n\times K}\) has one unit entry per row. Institution \(i\)'s error is \((Z_s\epsilon_s)_i\). An institution whose primary endpoint \(p_i\) survives remains there; an affected one uses an eligible survivor. The objective is
\begin{equation}
 F_s(Z_s)=\E[\norm{AZ_s\epsilon_s}_2^2\mid s]
       =\tr(AZ_sS_sZ_s^\top A^\top).
 \label{eq:F}
\end{equation}
This follows by expansion and requires no Gaussian assumption. Eligibility can encode operational or quality restrictions. Exact endpoint quotas add
\begin{equation}
 \sum_i Z_{s,ij}=q_{s,j}.
 \label{eq:quota}
\end{equation}
Equal quotas give identical endpoint counts across policies. They imply equal total inference cost only when cost per request is fixed within each endpoint.

A fractional \(Q_s\) satisfies the same linear constraints with entries in \([0,1]\). Its objective describes actual splitting or linear aggregation of decisions and is a convex relaxation of deterministic assignment. It does \emph{not} generally describe independently sampling one endpoint from each row.

Assignments must be chosen before evaluated errors are observed. Calibration data may determine a policy; held-out errors may not. Real outages can alter input and error distributions, so \(S_s\) need not equal the normal-state moment. Our endpoint-removal experiments hold the error distribution fixed to isolate reassignment.

\section{What primary diversification misses}
Consider scalar exposure of total mass one, equally spread over \(K\) primaries, with error second moment \(I_K\). Remove \(r\) endpoints, \(1\le r<K\), and retain surviving primary exposures.

\begin{theorem}[A contingent concentration floor]\label{thm:floor}
If all displaced mass uses a single surviving backup, post-removal risk is
\begin{equation}
 C_{\mathrm{common}}(K,r)=\frac{K+r(r+1)}{K^2}.
 \label{eq:common}
\end{equation}
If displaced exposure is divisible and every survivor is eligible, the minimum is
\begin{equation}
 C_{\mathrm{balanced}}(K,r)=\frac1{K-r}.
 \label{eq:balanced}
\end{equation}
For \(r/K\to\alpha\in(0,1)\), common-backup risk tends to \(\alpha^2\), balanced risk is asymptotic to \([K(1-\alpha)]^{-1}\), and their ratio is asymptotic to \(\alpha^2(1-\alpha)K\).
\end{theorem}
\begin{proof}
The common backup carries \((r+1)/K\); the other \(K-r-1\) survivors carry \(1/K\) each. Sum their squared loads. All survivor loads sum to one, so Cauchy--Schwarz gives the lower bound \(1/(K-r)\). Equal loads attain it because \(1/(K-r)\ge1/K\), requiring only displaced mass to move. The limits follow directly.
\end{proof}
The construction compares identical normal assignments and identical post-removal availability. It is a worst-case family, not a probability model for provider failures. With fixed \(r\) and \(K\to\infty\), the risk ratio tends to one; growing amplification requires outage size to grow with \(K\).

For unequal nonnegative survivor loads \(b_j\) and movable mass \(m\), minimizing \(\sum_j(b_j+y_j)^2\) over \(y_j\ge0\), \(\sum_jy_j=m\), gives the water-filling solution
\begin{equation}
 b_j+y_j=\max(b_j,\tau),\qquad \sum_j(\tau-b_j)_+=m.
 \label{eq:waterfill}
\end{equation}
The KKT conditions equate the marginal squared-load cost on endpoints receiving positive additions and leave higher initial loads unchanged. This is a continuous benchmark, not an integral algorithm.

\begin{proposition}[An irreducible common error]\label{prop:commonerror}
If \(S=(1-\gamma)I+\gamma\one\one^\top\), \(0\le\gamma\le1\), and \(Q\one=\one\), then
\begin{equation}
 F(Q)=(1-\gamma)\norm{AQ}_F^2+\gamma\norm{A\one}_2^2.
 \label{eq:equicorr}
\end{equation}
Routing changes only the first term. At \(\gamma=1\), every feasible routing has identical risk.
\end{proposition}
\begin{proof}
Substitute \(S\) into \eqref{eq:F} and use \(Q\one=\one\).
\end{proof}
For scalar unit mass, the second term is \(\gamma\). A nonzero common error therefore limits relative diversification gains even in the growing-outage construction.

\begin{figure}[t]
\centering\includegraphics[width=\linewidth]{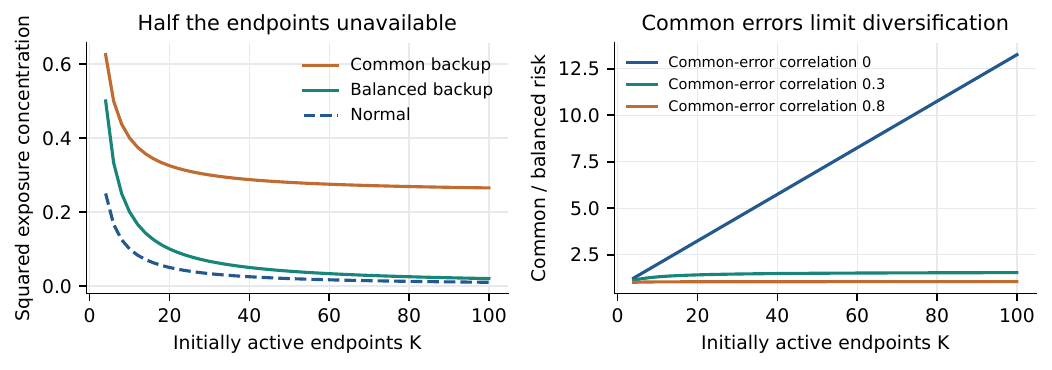}
\caption{Exact consequences of Theorem~\ref{thm:floor} with \(r=K/2\). Left: a shared backup leaves a nonzero concentration limit. Right: common error limits the relative improvement. These are analytical constructions, not estimates of outage likelihood.}
\end{figure}

\section{Indivisible decisions: the missing second moment}
Suppress the outage subscript. Independently for each institution, sample \(J_i\) from row \(q_i\) of \(Q\), independently of \(\epsilon\), and set \(Z_{ij}=\mathbf1\{J_i=j\}\).

\begin{theorem}[Exact random-routing premium]\label{thm:premium}
For every \(A\), row-stochastic \(Q\), and \(S\succeq0\),
\begin{align}
 \E_ZF(Z)&=F(Q)+D(Q),\label{eq:premium}\\
 D(Q)&=\sum_i\norm{a_i}_2^2
 \left(\diag(S)^\top q_i-q_i^\top S q_i\right)\ge0.
 \label{eq:D}
\end{align}
Here \(q_i\) is the column representation of row \(i\). Furthermore,
\begin{equation}
 D(Q)\le\lambda_{\max}(S)
       \sum_i\norm{a_i}_2^2(1-\norm{q_i}_2^2).
 \label{eq:boundD}
\end{equation}
\end{theorem}
\begin{proof}
Expand \(F(Z)=\sum_{i,\ell}(a_i^\top a_\ell)S_{J_iJ_\ell}\).
For \(i\ne\ell\), independence gives
\(\E S_{J_iJ_\ell}=q_i^\top S q_\ell\).
For \(i=\ell\), the expectation is \(\diag(S)^\top q_i\), whereas its counterpart in \(F(Q)\) is \(q_i^\top S q_i\). Subtraction proves the identity. The categorical covariance \(V_i=\diag(q_i)-q_iq_i^\top\) is positive semidefinite. Thus
\(0\le\tr(SV_i)\le\lambda_{\max}(S)\tr(V_i)
=\lambda_{\max}(S)(1-\norm{q_i}_2^2)\).
\end{proof}
The identity allows nonzero mean errors. It does not apply unchanged to shared random seeds, coordinated routing, or routing that observes contemporaneous errors: these change the joint assignment law.

\begin{corollary}[Effective-exposure granularity]\label{cor:neff}
Suppose \(S=I_K\), \(q_i=K^{-1}\one\), and \(A\one\ne0\). Define
\begin{equation}
 n_{\mathrm{eff}}(A)=\frac{\norm{A\one}_2^2}{\sum_i\norm{a_i}_2^2}.
 \label{eq:neff}
\end{equation}
Then
\begin{equation}
 \frac{\E_ZF(Z)}{F(Q)}=1+\frac{K-1}{n_{\mathrm{eff}}(A)}.
 \label{eq:relativepremium}
\end{equation}
If pairwise exposure inner products are nonnegative, \(1\le n_{\mathrm{eff}}(A)\le n\).
\end{corollary}
\begin{proof}
Here \(F(Q)=\norm{A\one}_2^2/K\) and
\(D(Q)=(1-1/K)\sum_i\norm{a_i}_2^2\).
Nonnegative pairwise inner products give the lower bound on \(n_{\mathrm{eff}}\); Cauchy--Schwarz gives the upper bound.
\end{proof}
For aligned \(a_i=w_iv\), with nonnegative weights summing to one, \(n_{\mathrm{eff}}=1/\sum_iw_i^2\). Large institutions can preserve a material premium despite a large headcount. If the columns of \(A\) are orthogonal, \(n_{\mathrm{eff}}=1\). With \(A=I_n\), every integral assignment has risk \(n\), while uniform fractional routing has risk \(n/K\). The integral-to-fractional ratio remains \(K\) for every \(n\).

\section{Routing procedures and certificates}
\subsection{Relaxation and finite-agent assignment}
Let \(\mathcal Q_s\) be the eligible row-stochastic polytope, optionally including quotas. Since
\begin{equation}
 F_s(Q)=\norm{AQS_s^{1/2}}_F^2,
 \label{eq:convex}
\end{equation}
minimization over \(\mathcal Q_s\) is a convex quadratic program.

\begin{proposition}[Conditional-expectation assignment]\label{prop:round}
For any feasible fractional \(Q\) with row-wise eligibility and no coupled quota constraints, an eligible deterministic \(Z\) can be constructed with
\begin{equation}
 F(Z)\le F(Q)+D(Q).
 \label{eq:round}
\end{equation}
Fix each remaining row to an eligible endpoint minimizing the exact conditional expectation of final risk.
\end{proposition}
\begin{proof}
Given previously fixed rows, the current expectation is the \(q_i\)-weighted average of expectations obtained by fixing the next row to each endpoint in its support. At least one is no greater than that average. Choosing a minimum cannot increase the expectation; repeat until every row is fixed.
\end{proof}
This specializes classical derandomization \citep{raghavan1987}. It does not preserve hard endpoint quotas, even if its input came from a quota-constrained relaxation.

\subsection{Quota-preserving exposure swaps}
For hard quotas, CER starts from count-balanced routing and swaps endpoints of two affected institutions. Unaffected institutions never move. Write \(M=AZ\). If \(i\) uses endpoint \(p\), \(\ell\) uses \(q\), and \(h=a_i-a_\ell\), the swap changes risk by
\begin{equation}
 \Delta F =
 2h^\top[(MS)_{:q}-(MS)_{:p}]
 +\norm{h}_2^2(S_{pp}+S_{qq}-2S_{pq}).
 \label{eq:swap}
\end{equation}
This follows by substituting \(M+h(e_q-e_p)^\top\) into the quadratic form. We take the best strictly improving eligible swap. Counts remain fixed. A scan of \(m\) affected institutions costs \(O(dm^2+dK^2)\) arithmetic with cached \(M\). Strict descent on a finite set ensures termination in exact arithmetic, but gives neither a polynomial iteration bound nor global optimality.

Even without quotas, scalar two-endpoint routing contains PARTITION. For positive integer exposures \(a_i\), \(S=I_2\), and total \(T\), risk is \(L^2+(T-L)^2\), where \(L\) is assigned to endpoint 1. It equals \(T^2/2\) exactly when the exposures admit an equal partition. This elementary reduction motivates reporting bounds rather than presenting local search as exact.

\subsection{A lower bound with the correct direction}
A feasible fractional objective is an upper bound on the fractional minimum; it cannot alone certify an integral gap. At any reference \(Q_0\), put \(G=2A^\top A Q_0S\). Convexity gives
\(F(Q)\ge F(Q_0)+\langle G,Q-Q_0\rangle\).
For equality quotas \(q_j\), choose any \(v_j\), and define
\begin{equation}
 \begin{split}
 u_i&=\min_{j\in\mathcal E_i}(G_{ij}-v_j),\\
 L&=F(Q_0)-\langle G,Q_0\rangle+\sum_i u_i+\sum_jq_jv_j .
 \end{split}\label{eq:dual}
\end{equation}
Then \(F(Q)\ge L\) for all feasible \(Q\), because \(u_i+v_j\le G_{ij}\) on eligible edges. We obtain \(v\) from the solver and recompute \(u_i\) to enforce this inequality. Without quotas, set \(v=0\). For \(L>0\), the reported gap is \((F(Z)-L)/L\). These are floating-point numerical bounds checked against small exhaustive problems, not interval-arithmetic certificates.

\subsection{Moment uncertainty and unequal endpoint quality}
For \(\widehat S\succeq0\) and
\(\mathcal U_\delta=\{S\succeq0:\norm{S-\widehat S}_2\le\delta\}\),
\begin{equation}
 \sup_{S\in\mathcal U_\delta}F(Q;S)
 =F(Q;\widehat S)+\delta\norm{AQ}_F^2
 =F(Q;\widehat S+\delta I).
 \label{eq:robust}
\end{equation}
Indeed \(C=Q^\top A^\top AQ\succeq0\) and
\(\tr[C(S-\widehat S)]\le\delta\tr(C)\), attained at
\(S=\widehat S+\delta I\). The envelope also applies before averaging randomized assignments. Choosing \(\delta\) needs justification; a selected ridge coefficient is not automatically a statistical confidence set.

Diversity is not always preferable. For two scalar endpoints with second moments \(s_1,s_2\) and cross moment \(c\), exposure weight \(w\) on endpoint 1 gives
\(f(w)=s_1w^2+s_2(1-w)^2+2cw(1-w)\).
If \(s_1+s_2-2c>0\), the optimum is
\begin{equation}
 w^*=\left[\frac{s_2-c}{s_1+s_2-2c}\right]_{[0,1]}.
 \label{eq:quality}
\end{equation}
In particular, \(c\ge s_1\) makes concentration at endpoint 1 optimal. CER optimizes financial error, not diversity for its own sake.

\section{Computational design}
\subsection{Synthetic portfolio networks}
The locally frozen protocol specifies 60 independently seeded networks with 48 institutions, six assets, and six endpoints. It was saved before collection, but was not externally preregistered. Institution sizes are proportional to independent \(\operatorname{Lognormal}(0,1.1^2)\) draws. Portfolio columns follow
\(H_{:i}\sim\operatorname{Dirichlet}(0.35\one_6)\).
Impact coefficients are proportional to \(\operatorname{Lognormal}(0,0.5^2)\) draws, normalized to mean one. For size weights \(w\) summing to one,
\begin{equation}
 E=\diag(\ell)H\diag(w),\quad
 T=EH^\top,\quad B=\beta T/\rho(T),\quad W=I.
 \label{eq:network}
\end{equation}
The feedback radii are \(\beta\in\{0,0.4,0.75\}\). These distributions create controlled heterogeneity; they are not fitted to balance sheets or market depth.

Each endpoint has eight randomly assigned primary institutions. All six single removals and all 15 double removals are enumerated. Error moments are \(S=(1-\gamma)I+\gamma\one\one^\top\), with \(\gamma\in\{0,0.3,0.8\}\).
The \(60\times3\times21\times3=11{,}340\) evaluations are dependent scenarios on 60 independent networks, not 11,340 independent observations.

Count-balanced routing processes affected institutions in a seeded random order, assigning each to a survivor with the smallest current count. CER starts from the same assignment and preserves its quotas. A common-backup comparator sends displaced institutions to the lowest-index survivor; equal second moments make this index arbitrary. Risks are evaluated by the exact second-moment formula, not estimated from simulated shocks.

\paragraph{Inference and controls.}
Within each network, we average \(1-F_{\rm CER}/F_{\rm count}\) over removal subsets of a given size. Percentile intervals resample the 60 network means 2,000 times. They summarize variability under the network generator, not uncertainty about the actual financial system. Lower bounds are evaluated at \(\beta=0.4\). Homogeneous exposures and the \(\gamma=1\) identity provide zero-benefit controls.

The implementation has 48 automated checks: assignment enumeration, lower-bound comparisons, iterative cascades, exact accounting identities, quota and eligibility preservation, and negative controls. A separate study uses \(n\in\{4,8,16,32,64\}\), four endpoints, and 20 seeds per size to compare the fractional lower bound, exact independent-routing risk, and realized conditional-expectation assignment.

\subsection{Recorded API errors and held-out replay}
We also collect errors rather than prescribe \(S\). The task is one-asset rebalancing with proportional fees charged against post-trade net asset value (NAV). Given initial NAV \(V\), risky position \(A_0\), target fraction \(w\), and fee rate \(c\), signed trade \(u\) satisfies
\begin{equation}
 A_0+u=w(V-c|u|),\qquad
 u=\frac{wV-A_0}{1+wc\,\operatorname{sgn}(wV-A_0)}.
 \label{eq:rebalance}
\end{equation}
The generator uses four wordings, positive and negative gaps, and targets that can exceed one; borrowing and shorting are expressly allowed. Denominators are positive throughout the declared range. A 40-digit decimal calculation supplies ground truth.

There are 256 tasks and four pinned endpoints:
\begin{itemize}
\item \texttt{gpt-4.1-nano-2025-04-14};
\item 
\texttt{gpt-4.1-mini-2025-04-14};
\item 
\texttt{gpt-4o-mini-2024-07-18}; and
\item 
\texttt{gpt-4.1-2025-04-14}.
\end{itemize}
Every task-endpoint pair is requested once at temperature zero, with JSON-object output and at most 180 completion tokens. A transport failure permits one retry. Invalid or non-stop outputs execute a zero trade under the declared fallback. Raw text, returned model identifiers, usage, timestamps, and response identifiers are retained.

A seeded split fixes 128 calibration and 128 test tasks before collection. Define \(e_{tj}=(\widehat u_{tj}-u_t)/V_t\) and fit only on calibration tasks:
\[
 \widehat S=128^{-1}\sum_{t\in{\rm cal}}e_te_t^\top.
\]
The same 60 synthetic exposure networks are instantiated with four endpoints. All single and double endpoint omissions are evaluated. A fixed assignment for each network and omission is applied to every held-out error vector. This is a \emph{shared-advice replay}: recorded errors are reused across assigned institutions. It does not measure the dependence of separately prompted agents or real vendor outages.

The primary comparator preserves endpoint counts. Secondary diagnostics replace \(\widehat S\) by its diagonal, the identity, or \(\widehat S+0.1\,\tr(\widehat S)I/K\). These were specified after collection and are exploratory. A common-best comparator sends displaced institutions to the survivor with smallest calibration diagonal moment; its endpoint counts differ.

Held-out task losses are averaged over the fixed network/scenario collection. We then resample 128 tasks with model responses paired. The reduction is the ratio-of-means contrast \(1-\overline L_{\rm CER}/\overline L_{\rm count}\), unlike the synthetic mean-of-ratios statistic. The 2,000 bootstrap resamples keep fitted policies fixed and condition on the calibration split.

\section{Results}
\subsection{Exposure-aware assignment improves the matched-count comparison}
At \(\beta=0.4,\gamma=0\), mean reductions are 6.57\% for single removals and 10.53\% for double removals (Table~\ref{tab:synthetic}). The intervals exclude zero. This is consistent with in-sample monotonicity: the intervals describe effect size across networks, not a test that a descending optimizer descends. Every endpoint quota is preserved.

Mean gaps to numerical lower bounds are 2.49\% and 2.73\%. These quantify room for improvement relative to the relaxation, not exact integral optimality. The largest primal constraint residual is \(1.31\times10^{-8}\). Median swap runtime is approximately 0.14 milliseconds on an Apple M4 Pro, a small-instance implementation measurement rather than a production throughput benchmark.

\begin{table}[t]
\centering\small
\caption{Synthetic results at \(\beta=0.4\). Reductions and gaps are percentages. Intervals resample 60 network-level mean reductions; all removal subsets of each size are included.}
\label{tab:synthetic}
\begin{tabular}{ccrrr}
\toprule
Error correlation & Removed & Mean reduction & 95\% interval & Mean gap to bound\\
\midrule
0.0 & 1 & 6.57 & [5.99, 7.17] & 2.49 \\
0.0 & 2 & 10.53 & [9.93, 11.14] & 2.73 \\
0.3 & 1 & 2.58 & [2.33, 2.84] & 0.84 \\
0.3 & 2 & 4.54 & [4.23, 4.84] & 1.03 \\
0.8 & 1 & 0.43 & [0.38, 0.48] & 0.13 \\
0.8 & 2 & 0.79 & [0.74, 0.85] & 0.17 \\
 
\bottomrule
\end{tabular}
\end{table}

\begin{figure}[t]
\centering\includegraphics[width=\linewidth]{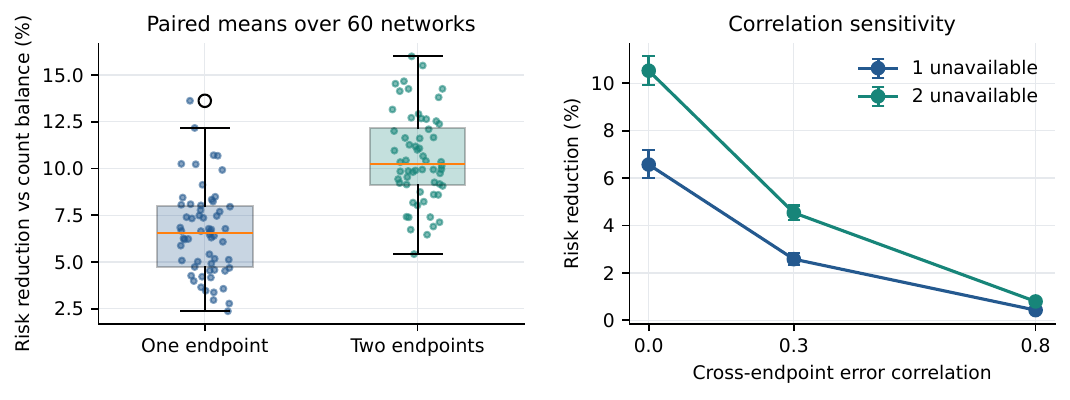}
\caption{Matched-count comparison. Left: one point per network, averaged over removals of the given size; \(\beta=0.4,\gamma=0\). Right: means and network bootstrap intervals as common error increases. The common component is invariant to routing.}
\end{figure}

At correlation 0.8 the benefits fall to 0.43\% and 0.79\%. Homogeneous exposures give exactly equal risks for matched quotas. These negative controls show that the method cannot remove common error or improve an already identical financial allocation. Feedback variation preserves positive average reductions but changes their magnitude (Appendix~\ref{app:details}).

\subsection{Fractional policies can overstate diversification}
Across 20 seeds, independent routing exceeds the fractional lower bound by an average 176.0\% at \(n=4\) and 28.0\% at \(n=64\). Conditional-expectation assignments reduce the gaps to 134.7\% and 5.0\%. Every realized assignment satisfies its finite-agent bound. Figure~\ref{fig:round} also shows the exact orthogonal-exposure counterexample, where increasing \(n\) does not close the gap.

\begin{figure}[t]
\centering\includegraphics[width=\linewidth]{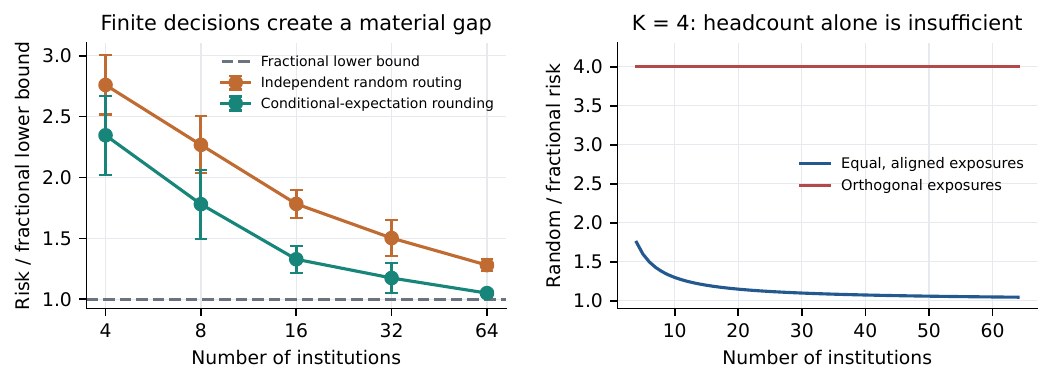}
\caption{Four-endpoint indivisible-routing effects. Left: 20-seed means with descriptive 1.96-standard-error bars, normalized by the fractional lower bound. The dashed line is a relaxation, not an executable random policy. Right: exact uniform-routing cases; orthogonal exposures retain a factor-four gap at every population size.}
\label{fig:round}
\end{figure}

\subsection{Recorded-error replay: smaller, conditional benefits}
Collection yielded 1,024 observations and one transport retry, for 1,025 primary HTTP requests. Fifty-nine responses reached the completion limit and used zero-trade fallback. Estimated primary token cost was \$0.2143 at nominal rates; this is not an invoiced charge.

Held-out replay gives a 3.30\% reduction against matched-count routing, with paired interval 2.07--4.57\% (Table~\ref{tab:api}). Single- and double-omission estimates are 2.98\% and 3.52\%. Calibration error alignment \(S_{jk}/\sqrt{S_{jj}S_{kk}}\) ranges from 0.848 to 0.999 off diagonal. This is an uncentered cosine, not a Pearson correlation. Strong alignment is consistent with smaller gains than in the independent-error construction.

\begin{table}[t]
\centering\small
\caption{Held-out API replay. The first four policies preserve baseline quotas. Common best changes counts and uses the survivor with smallest calibration second moment. Intervals condition on fitted policies; exploratory comparisons have no multiplicity adjustment.}
\label{tab:api}
\begin{tabular}{lrr}
\toprule
Assignment objective & Reduction (\%) & 95\% interval\\
\midrule
CER (full second moment) & 3.30 & [2.07, 4.57] \\
Diagonal second moment & 1.20 & [0.74, 1.67] \\
Equal-variance exposure balance & 0.13 & [-0.04, 0.38] \\
Ridge-envelope diagnostic & 2.85 & [1.77, 3.97] \\
Common best (different counts) & 4.53 & [2.75, 6.32] \\
 
\bottomrule
\end{tabular}
\end{table}

Common best achieves a larger 4.53\% reduction with different counts. This does not contradict CER's constrained objective; it shows the importance of quality and resource constraints. Equal-variance exposure balancing gives only 0.13\% overall improvement, with an interval crossing zero. Ignoring the estimated error structure loses most of the benefit in this replay.

\begin{figure}[t]
\centering\includegraphics[width=\linewidth]{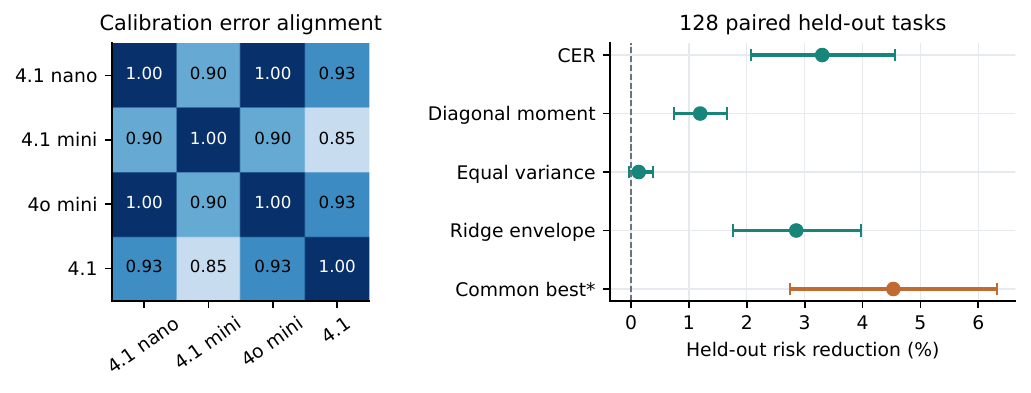}
\caption{Recorded-error replay. Left: uncentered calibration error alignment; rounding to 1.00 does not imply exact equality. Right: held-out reductions and paired task intervals. The asterisk denotes different endpoint counts. All four pinned model versions use one vendor.}
\end{figure}

\paragraph{Validity limit of the API evidence.}
The one-shot arithmetic performance is poor (Table~\ref{tab:models}); these runs are unsuitable as financial execution engines. Post-collection diagnostics increase the allowance to 1,000 tokens and use four transparent prompts per endpoint. All four endpoints answer \(2+2\) correctly; two solve a no-fee rebalancing example; none returns the specified cent-accurate result in the eight explicit-formula and worked-formula probes. The 16 records are saved separately. They do not explain the poor performance, but show that the primary completion limit is not the only issue.

The exact deterministic implementation of \eqref{eq:rebalance} solves the task with zero unrounded arithmetic error and dominates every model-only policy. The replay therefore supports the routing calculation on the recorded error panel, not a representative capability ranking, a case for LLM arithmetic, or a claim that deployed financial agents follow this error law. The mathematical and synthetic results do not depend on this panel.

\section{Operational interpretation and limitations}
CER applies most directly to a platform that knows the financial exposure attached to a decision, has approved several eligible endpoints, and can change routing during a disruption. A practical interface accepts holdings or response estimates, a validation error panel, eligibility, and capacities; it returns executable assignments, stress metrics, and relaxation gaps. This can support vendor-contingency review or an internal routing control. A commercial deployment would need reliable exposure data, contracted capacity, state-dependent validation, monitoring, and governance beyond this prototype.

Several boundaries prevent stronger conclusions. All balance sheets and impact coefficients here are synthetic. We estimate neither actual AI adoption nor real financial concentration. The shared-error model omits private prompts and institution-specific residuals. A routing-independent residual second moment uncorrelated with endpoint errors adds a constant; endpoint-dependent or institution-dependent residual quality requires an extended objective.

Second moments do not order tail probabilities for arbitrary distributions. The Gaussian and standardized Student-\(t_5\) checks in Appendix~\ref{app:tail} illustrate sensitivity, not a tail-risk theorem. Error moments can drift after an outage; the spectral envelope is a sensitivity device, not automatic statistical coverage. The quota heuristic gives local guarantees, and the lower bound can be loose for granular exposures. Finally, all measured endpoints use one API vendor. Column omission is not a simulation of independent cloud-provider failures.

\section{Conclusion}
Failover changes the joint allocation of financial errors. Primary endpoint counts and availability do not characterize the resulting concentration. A financial response matrix makes this dependence explicit, while the exact random-routing premium separates executable single-endpoint decisions from fractional allocation. The experiments show exposure-allocation gains under matched counts and sharply smaller gains for common errors. The result is a reproducible method for evaluating contingent routing under declared assumptions. Empirical calibration and independent validation remain necessary before using it to set operational risk limits.

\paragraph{Code and data availability.}
The \href{https://github.com/shi1720/contingent-exposure-routing}{companion repository on GitHub} contains the Python implementation, tests, locally frozen protocol, constructed prompts, raw responses, processed moments, numerical results, figure scripts, and manuscript source. Offline reproduction makes no API calls. New collection requires an explicit paid-execution flag and a separately supplied credential. No credential is included. Results refer to the archived panel; a remote model rerun need not return identical text.

\paragraph{Declarations.}
OpenAI Codex assisted with literature review, mathematics, code, experiments, and drafting. The author assumes responsibility for all content.

\bibliographystyle{plainnat}
\bibliography{references}
\clearpage
\appendix

\section{Additional experimental details}\label{app:details}
\subsection{Task distribution and model outcomes}
NAV is uniform on \$1,000 increments from \$20,000 to \$500,000. Initial risky exposure ranges from 5\% to 145\% of NAV in one-point increments; targets range from 10\% to 140\%. Fees are uniform on \(\{5,15,35,75,150,300\}\) basis points. Negative cash denotes borrowing. Wording templates cycle by index; the calibration/test permutation is independent of outcomes.

\begin{table}[h]
\centering\small
\caption{Held-out one-shot calculations. Errors are basis points of initial NAV after zero-trade fallback. Valid means a parsed finite number with normal completion, not a correct answer. Poor task performance limits the replay's interpretation.}
\label{tab:models}
\begin{tabular}{lrrrr}
\toprule
Endpoint & Valid & RMSE (bp) & MAE (bp) & Within 1 bp\\
\midrule
GPT-4.1-nano & 123/128 & 5211.1 & 4218.1 & 0/128 \\
GPT-4.1-mini & 128/128 & 4563.3 & 3802.0 & 0/128 \\
GPT-4o-mini & 127/128 & 5345.8 & 4293.4 & 1/128 \\
GPT-4.1 & 107/128 & 4484.6 & 3462.6 & 0/128 \\
 
\bottomrule
\end{tabular}
\end{table}

Post-collection diagnostics are an integrity check, not an alternative benchmark selected for better results. Their prompts, expected values, responses, finish reasons, and identifiers are in \path{data/raw/api-sanity.jsonl}. Primary data are unchanged. Primary and diagnostic requests total 1,041 HTTP calls; the quoted primary cost excludes the 16 diagnostics.

\subsection{Feedback and granularity}
At zero endpoint correlation, mean single/double-removal reductions are 7.30\%/12.33\% for \(\beta=0\), 6.57\%/10.53\% for \(\beta=0.4\), and 6.35\%/9.36\% for \(\beta=0.75\). Greater feedback need not increase the relative gain, because it changes the exposure geometry and baseline risk. Absolute risk depends on feedback and participation scale.

The granularity study uses the same size and portfolio generator, no fixed primary assignments, no quotas, four eligible endpoints, and \(\beta=0.4\). It isolates implementability. Continuous and independent-random risks are analytical; error bars in Figure~\ref{fig:round} summarize seed variation rather than shock simulation uncertainty.

\section{Secondary tail checks}\label{app:tail}
A secondary simulation uses network seed 1721406, \(\beta=0.4\), and removal of endpoints 0 and 1. At each \(\gamma\in\{0,0.3,0.8\}\), it draws 100,000 common random vectors for all policies. Gaussian vectors have second moment \(S\). Dividing them by \(\sqrt{U/3}\), with independent \(U\sim\chi^2_5\), gives Student vectors with the same second moment. These checks were added after the primary evaluation and are diagnostic.

The tail quantity is \(\left|d^{-1}\one^\top x\right|\), absolute mean asset displacement in normalized units. Empirical ES97.5 is the mean at or above its empirical 97.5\% quantile. It is not expected shortfall of a calibrated investment portfolio.

\begin{figure}[h]
\centering\includegraphics[width=.88\linewidth]{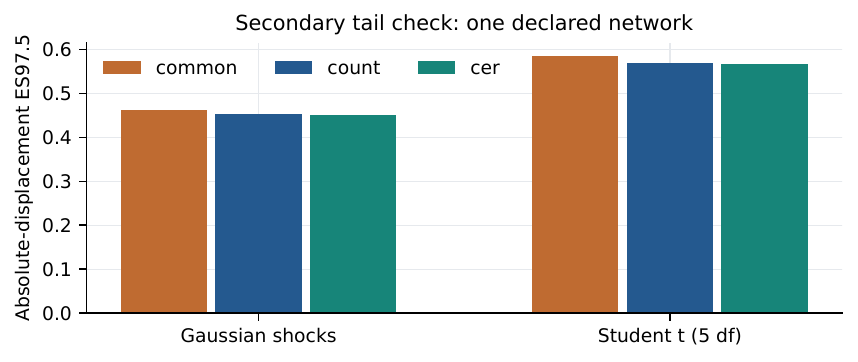}
\caption{Secondary absolute-displacement tail check for one declared network, with independent errors. Heavy tails raise ES despite unchanged second moments. CER's ES improvement here is smaller than its quadratic-risk improvement.}
\end{figure}
At \(\gamma=0\), count-balanced and CER theoretical quadratic risks are 0.41745 and 0.39745. Gaussian estimates are 0.41809 and 0.39777, with standard errors 0.00177 and 0.00172. Their ES values are 0.45436 and 0.45142 under Gaussian shocks, versus 0.56979 and 0.56726 under Student shocks. The complete table includes all correlations.

\section{Independent checks and reproduction}
The exhaustive checker evaluates
\(\sum_{i,\ell}(a_i^\top a_\ell)S_{J_iJ_\ell}\) directly over all \(3^4\) assignments on 12 random instances, without reusing the reduced random-routing formula. Ten small problems enumerate \(3^6\) assignments subject to a fixed row and compare the lower bound to the best integral objective. Another example enumerates all equal-quota assignments. Accounting tests substitute decimal ground truth into the original NAV equation for all 256 tasks.

The resolvent is checked against 1,000 cascade iterations at radii 0, 0.4, 0.75, and 0.95. Swap tests independently scan terminal assignments for improving legal pairs and verify fixed rows, endpoint counts, and objective changes. Other checks cover perfectly common errors, orthogonal exposures, the granularity identity, and the spectral envelope. These are implementation checks, not substitutes for proofs.

Principal offline commands are:
\begin{lstlisting}[basicstyle=\ttfamily\small,breaklines=true,frame=single]
python -m pytest -q
python scripts/run_synthetic.py
python scripts/analyze_api.py
python scripts/run_tail_checks.py
python scripts/summarize.py
python scripts/make_tables.py
python scripts/make_figures.py
cd paper && tectonic main.tex
\end{lstlisting}
The locked environment, seeds, generated tables, and checksums accompany the artifact. Synthetic results are deterministic conditional on the numerical environment; solver tolerances can cause small platform-dependent differences. Bootstrap intervals describe only their stated resampling distributions. They are not formal finite-sample guarantees or probabilities that the modeling assumptions hold in practice.
\end{document}